\documentclass{article}
\usepackage{ijcai18}

\usepackage{times}
\usepackage{xcolor}
\usepackage{soul}
\usepackage[utf8]{inputenc}
\usepackage[small]{caption}

\usepackage{amsmath,amssymb}
\usepackage{amsthm}
\usepackage{graphicx}
\newtheorem{theorem}{Theorem}

\title{Pretraining Deep Actor-Critic Reinforcement Learning Algorithms \\
With Expert Demonstrations}

\author{
Xiaoqin Zhang, 
Huimin Ma
\\ 
Dept. of EE, Tsinghua University \\
xiaoqin-15@mails.tsinghua.edu.cn,
mhmpub@tsinghua.edu.cn
}

\begin{document}

\maketitle

\begin{abstract}
Pretraining with expert demonstrations have been found useful in speeding up the training process of deep reinforcement learning algorithms since less online simulation data is required.
Some people use supervised learning to speed up the process of feature learning, others pretrain the policies by imitating expert demonstrations. However, these methods are unstable and not suitable for actor-critic reinforcement learning algorithms. Also, some existing methods rely on the global optimum assumption, which is not true in most scenarios.
In this paper, we employ expert demonstrations in a actor-critic reinforcement learning framework, and meanwhile ensure that the performance is not affected by the fact that expert demonstrations are not global optimal. We theoretically derive a method for computing policy gradients and value estimators with only expert demonstrations. 
Our method is theoretically plausible for actor-critic reinforcement learning algorithms that pretrains both policy and value functions.
We apply our method to two of the typical actor-critic reinforcement learning algorithms, DDPG and ACER, and demonstrate with experiments that our method not only outperforms the RL algorithms without pretraining process, but also is more simulation efficient.

\end{abstract}

\section{Introduction}

Deep reinforcement learning is a general method that have been successful in solving complex control problems. Mnih et al. in \cite{mnih2015human} combined Q learning with deep neural networks and proved to be successful in image based Atari games.

Policy gradient methods have been proved significantly efficient in both continuous control problems (\cite{sutton1999}, \cite{silver2014deterministic}, \cite{heess2015learning}) and discrete control problems (\cite{silver2016mastering}, \cite{Wang2016}). Among policy gradient methods, actor-critic algorithms are at the heart of many significant advances in reinforcement learning (\cite{bhatnagar2009natural}, \cite{Degris2012}, \cite{Lillicrap2015}, \cite{mnih2016asynchronous}). These algorithms estimate state-action value functions independently,  and proved to be efficient in policy optimization.

However, an enormous number of online simulation data is required for deep reinforcement learning. Hence we attempt to learn from expert demonstrations and decrease the amount of online data required in deep reinforcement learning algorithms.

One of the representative method of learning from expert demonstrations is inverse reinforcement learning. Ng et al. proposed the first inverse reinforcement learning algorithm \cite{Ng2000}, which recovers reward function based on the assumption that the expert policy is the global optimal policy. From recovered reward function, Abbeel et al. are able to propose apprenticeship learning (\cite{abbeel2004apprenticeship}) to train a policy with expert demonstrations and a simulation environment that does not output reward. Apprenticeship learning inspired many similar algorithms (\cite{syed2008game}, \cite{syed2008apprenticeship}, \cite{piot2014boosted}, \cite{ho2016model}), Ho et al. \cite{ho2016generative} proposed a imitation learning  method that merges inverse reinforcement learning and reinforcement learning, hence imitate the expert demonstrations with generative adversarial networks (GANs).

These algorithms proved successful in solving MDP$\backslash$R (\cite{abbeel2004apprenticeship}). However, MDP$\backslash$R is different from original MDP since MDP$\backslash$R environments do not output task based reward data. And for this reason, inverse reinforcement based algorithms attempt to assume the expert demonstrations to be global optimal and imitate the expert demonstrations. In order to learn from expert demonstrations for MDP, alongside with state-of-the-art reinforcement learning algorithms, different frameworks are required.

There are some prior work that attempt to make use of expert demonstrations for reinforcement learning algorithms. Lakshminarayanan et al. \cite{Lakshminarayanan2016} proposed a training method for DQN based on the assumption that expert demonstrations are global optimal, thus pretrain the state-action value function estimators.

Cruz Jr et al. \cite{cruz2017pre} focused on feature extracting for high dimensional, especially image based simulation environments, and proposed a framework for discrete control problems that pretrains the neural networks with classification tasks using supervised learning. The purpose of this pretraining process is to speed up the  training process by trying to extract features of high dimensional states. However, this work is only suitable for image based, discrete action environments, and ignored the fact that expert demonstrations perform better than current learned policies.

The first published version of AlphaGo \cite{silver2016mastering} is one of the most important work that pretrains the neural networks with human expert demonstrations. In this work, a policy network and a value network is used. The value network is trained with on-policy reinforcement learning, and the policy network is pretrained with expert demonstrations using supervised learning, then trained with policy gradient. This work and \cite{cruz2017pre} are quite similar, the role of expert demonstrations is to speed up the feature extraction, and to give policy a warm start. The fact that expert demonstrations perform better is not fully used, and the framework is not extensive enough for other problems and other reinforcement learning algorithms.

In this paper, we propose an extensive framework that pretrains actor-critic reinforcement learning algorithms with expert demonstrations, and use expert demonstrations for both policy functions and value estimators.
We theoretically derive a method for computing policy gradient and value estimators with only expert demonstrations. Experiments show that our method improves the performance of baseline algorithms on both continuous control environments and high-dimensional-state discrete control environments. 


\section{Background and Preliminaries}
In this paper, we deal with an infinite-horizon discounted Markov Decision Process (MDP), which is defined by the tuple $\{S, A, P, r, \rho_0, \gamma\}$. In the tuple, $S$ is a finite set of states, $A$ is a finite set of actions, $P: S\times A \times S \rightarrow \mathbb{R}$ is the transition probability distribution, $r: S \rightarrow \mathbb{R}$ is the reward function, $\rho_0 : S\rightarrow \mathbb{R}$ is the probability distribution of initial state $S_0$, and $\gamma \in (0,1)$ is the discount factor.

A stochastic policy $\pi^s:S\times A \rightarrow \mathbb{R}$ returns the probability distribution of actions based on states, and a deterministic policy $\pi^d: S \rightarrow A$ returns the action based on states. In this paper, we deal with both stochastic policies and deterministic policies, and $a\sim \pi(s)$ means $a\sim \pi^s(a|s)$ or $a = \pi^d(s)$ respectively. Thus the state-action value function $Q^\pi$is:

\begin{equation*}
Q^\pi(s_t,a_t)=\mathbb{E}_{s_{t+1},a_{t+1},...}\left[\sum_{\tau=0}^{\infty}\gamma^\tau r(s_{t+\tau})\right]
\end{equation*}

The definitions of the value function $V^\pi$ and the advantage function $A^\pi$ are:
\begin{gather*}
V^\pi(s_t)=\mathbb{E}_{a_t,s_{t+1},...}\left[\sum_{\tau=0}^{\infty}\gamma^\tau r(s_{t+\tau})\right] \\
A^\pi(s_t,a_t)=Q^\pi(s_t,a_t)-V^\pi(s_t) \\
\end{gather*}

And let $\eta(\pi)$ denote the discounted reward of $\pi$:

\begin{equation*}
\eta(\pi)=\mathbb{E}_{s_{0},a_{0},...}\left[\sum_{t=0}^{\infty}\gamma^t r(s_{t})\right]
\end{equation*}

For future convenience, let $d^\pi(s)$ denote the limiting distribution of states:
\begin{equation*}
d^\pi(s)=\lim_{t\rightarrow\infty}Pr(s_t=s)
\end{equation*}

where in all of the definitions above:
\begin{equation*}
s_0\sim \rho_0(s_0), a_t\sim \pi(s_t), s_{t+1}\sim P(s_{t+1}|s_t,a_t)
\end{equation*}

The goal of actor-critic reinforcement learning algorithms is to maximize the discounted reward, $\eta(\pi)$, to obtain the optimal policy, where we use a parameterized policy $\pi_\theta$. While estimating $\eta(\pi)$ or $\nabla_\theta\eta(\pi_\theta)$ based on simulated samples, many algorithms use a state-action value estimator $Q^w$, to estimate the state-value function $Q^\pi$ for policy function $\pi_\theta$. 

One typical deterministic actor-critic algorithm DDPG (Deep Deterministic Policy Gradient) \cite{Lillicrap2015} uses estimator $Q^w=\hat{Q^\pi}$ to estimate the gradient of an off-policy deterministic discounted reward $\eta_\beta(\pi_\theta)=\sum_{s\in S}d^\beta(s)V^\pi(s)$ \cite{Degris2012}, where $\beta$ is the roll-out policy:

\begin{equation*}
\begin{split}
\nabla_\theta\eta_\beta(\pi)&\approx \mathbb{E}_{s_t\sim d^\beta(s)}\left[\nabla_\theta Q^w(s,a)|_{s = s_t, a\sim \pi_\theta(s_t)}\right]\\
&=\mathbb{E}_{s_t\sim d^\beta(s)}\left[\nabla_{a} Q^w(s,a)|_{s=s_t, a\sim \pi_\theta(s_t)}\pi_\theta'\right]
\end{split}
\end{equation*}

Where $Q^w$ is updated with sampled data from $\pi$ using Bellman equation, $\pi_\theta'=\nabla_\theta\pi_\theta(s)|_{s=s_t}$.

Another off-policy algorithm that has $Q^w$ as an estimator of policy $\pi_\theta$ is ACER (Actor-Critic with Experience Replay) \cite{Wang2016} that optimizes stochastic policy. The algorithm maximizes off-policy deterministic discounted reward $\eta_\beta(\pi_\theta)$ as well, and modifies the off-policy policy gradient $\hat{g}^{acer}=\nabla_\theta\eta_\beta(\pi)$ to:

\begin{equation*}
\begin{split}
& \hat{g}^{acer}= \bar{\rho}_t\nabla_\theta \log\pi_\theta(a_t|s_t)\left[Q^{ret}(s_t,a_t)-V^w(s_t)\right]\\
&+\mathbb{E}_{a\sim \pi_\theta(s_t)}\left(\left[\frac{\rho_t(a)-c}{\rho_t(a)}\right]_+\nabla_\theta\log\pi_\theta(a|s_t)A^w(s_t,a)\right)
\end{split}
\end{equation*}
Where $A^w(s_t,a)=Q^w(s_t,a)-V^w(s_t)$, $\bar{\rho}_t=\min\left\{c,\frac{\pi(a_t,s_t)}{\beta(a_t,s_t)}\right\}$; $\left[x\right]_+=x$ if $x>0$ and is zero otherwise; $V^w(s_t)=\mathbb{E}_{a\sim\pi_\theta(s_t)}(s_t,a)$; $\rho_t(a)=\frac{\pi(a,s_t)}{\beta(a,s_t)}$; $s_t\sim d^\beta(s_t)$ and $a_t\sim \beta(s_t)$; $Q^{ret}$ is the Retrace estimator of $Q^\pi$ \cite{Munos2016}, which can be expressed recursively as follows:

\begin{equation*}
\begin{split}
Q^{ret}(s_t,a_t)=
r_t+\gamma\bar{\rho}_{t+1}\delta_Q(s_{t+1},a_{t+1})+\gamma V^w(s_{+1})
\end{split}
\end{equation*}
where
\begin{equation*}
\delta_Q(s_{t+1},a_{t+1}) = Q^{ret}(s_{t+1},a_{t+1})-Q^w(s_{t+1},a_{t+1})
\end{equation*}

In ACER, state-action value function is updated using $Q^{ret}$ as target, with gradient $g_Q$:

\begin{equation*}
\begin{split}
g_Q=(Q^{ret}(s_t,a_t)-Q^w(s_t,a_t))\nabla_wQ^w(s_t,a_t)
\end{split}
\end{equation*}

In this paper, we will apply our methods with expert demonstrations to DDPG and ACER. 

\section{Expert Based Pretraining Methods}

Suppose there exists an expert policy $\pi^*$ that performs better than $\pi$. We define \emph{perform better} with the following straightforward constraint:

\begin{equation}\label{cons}
\eta(\pi^*) \geqslant \eta(\pi)
\end{equation}

The definition of \emph{perform better} above is based on the fact that the goal of actor-critic RL algorithms is to maximize $\eta(\pi)$. Here the expert policy $\pi^*$ is different from that of IRL \cite{Ng2000}, imitation learning \cite{Ho2016} or LfD \cite{hester2017learning}, since $\pi^*$ here is \emph{not} the optimum policy of the MDPs.

Here we define a demonstration of a policy $\pi$ as a sequence of $(s,a)$ pairs, $\{(s_t,a_t)\}_{t=0,1,2,...}$, sampled from $\pi$.

Actor-critic RL algorithms tend to optimize $\eta(\pi_\theta)$ as the target. Thus pretraining procedures for these algorithms need to estimate $\eta(\pi_\theta)$ as the optimization target using expert demonstrations. Also, from definition (\ref{cons}), we need to estimate $\eta(\pi^*)$ as well.

However, With only demonstrations of expert policy $\pi^*$ and a black-box simulation environment, $\eta(\pi^*)$ and $\eta(\pi_\theta)$ cannot be directly estimated. Hence we introduce Theorem \ref{theo} (see \cite{Schulman2015} and \cite{Kakade2002}).

\begin{theorem}\label{theo}
For two policies $\pi$ and $\pi^*$:
\begin{equation} \label{theoeq}
\eta(\pi^*)-\eta(\pi) = \mathbb{E}_{s^*_0,a^*_0,...\sim \pi^*}\left[\sum_{t=0}^{\infty}\gamma^t A^\pi(s^*_t,a^*_t)\right]
\end{equation}
\end{theorem}
\begin{proof}
(See also \cite{Schulman2015} and \cite{Kakade2002}) Note that 
\begin{equation*}
A^\pi(s,a)=\mathbb{E}_{s'\sim P(s'|s,a)}\left[r(s)+\gamma V^\pi(s')-V^\pi(s)\right]
\end{equation*}
we have:
\begin{equation*}
\begin{split}
&\mathbb{E}_{s^*_0,a^*_0,...\sim \pi^*}\left[\sum_{t=0}^{\infty}\gamma^t A^\pi(s^*_t,a^*_t)\right]\\
=&\mathbb{E}_{s^*_0,a^*_0,...\sim \pi^*}\left[\sum_{t=0}^\infty \gamma^t(r(s_t)+\gamma V^\pi(s_{t+1})-V^\pi(s_t))\right]\\
=&-\mathbb{E}_{s^*_0\sim \rho_0}\left[V^\pi(s_0)\right]+\mathbb{E}_{s^*_0,a^*_0,...\sim \pi^*}\left[\sum_{t=0}^\infty\gamma^tr(s_t)\right]\\
=&-\eta(\pi)+\eta(\pi^*)
\end{split}
\end{equation*}

\end{proof}

For many actor-critic RL algorithms like DDPG and ACER, policy optimization is based on accurate estimations of state-action value functions or value functions of the learned policy $\pi_\theta$.  Typically, those algorithms use data sampled from $\pi_\theta$, $\{(s_t,a_t,r_t)\}_{t=0,1,2,...}$, to estimate $Q^\pi$ and $V^\pi$. The estimating processes usually need a large amount of simulations to be accurate enough.

Combine Theorem \ref{theo} with constraint (\ref{cons}), we have:

\begin{equation}\label{finalcons}
\mathbb{E}_{s^*_0,a^*_0,...\sim \pi^*}\left[\sum_{t=0}^{\infty}\gamma^t A^\pi(s^*_t,a^*_t)\right]\geqslant 0
\end{equation}

This result links state-action value functions with expert demonstrations, allowing us to apply constraint (\ref{cons}) while training state-action value functions. This constraint is for value estimators, like $Q^w$ and $V^w$. When value estimators are not accurate enough, constraint (\ref{finalcons}) would not be satisfied. Hence if an algorithm update value estimators under constraint (\ref{finalcons}), the estimators would be more accurate, and in result improve the policy optimizing process.

Another pretraining process is policy optimization using expert demonstrations. Like most actor-critic algorithms, we suppose advantage function $A^\pi(s,a)$ is already known while conducting policy optimization. Then we can estimate the update step with expert demonstrations and estimations of value functions.

Considering Theorem \ref{theo}, we estimate he policy gradient as the following:
\begin{equation}\label{policypre}
\begin{split}
&\nabla_\theta \eta(\pi_\theta)\\
=&\nabla_\theta (\eta(\pi_\theta)-\eta(\pi^*))\\
=&-\nabla_\theta\mathbb{E}_{s^*_0,a^*_0,...\sim \pi^*}\left[\sum_{t=0}^{\infty}\gamma^t A^\pi(s^*_t,a^*_t)\right]
\end{split}
\end{equation}

Equation (\ref{policypre}) provides an off-policy policy optimization procedure with data only from expert demonstrations. It turns out that \emph{perform better} is not a must in this procedure for expert policy $\pi^*$. 

Recently, people like to propose sample efficient RL algorithms, like ACER and Q-Prop \cite{Gu2017}, since RL algorithms need a large amount of simulation time while training. With expert demonstrations, since there is no reward data, we cannot conduct sample efficient policy optimization processes. However, when we update policies with (\ref{policypre}), no simulation time is needed. We call  the situation \emph{simulation efficient}, which means the algorithms may need a large amount of data, but need few simulation data while training.

Note that sample efficient algorithms are all simulation efficient algorithms, all of these methods intend to decrease the simulation time. In this paper, we evaluate our method by how simulation efficient it is.

In this section, we found two pretraining methods for actor-critic RL algorithms, namely (\ref{finalcons}) and (\ref{policypre}). Both of them are based on Theorem \ref{theo}. The theorem connects policy discounted reward $\eta(\pi_\theta)$ and expert demonstration data, requiring no reward data from expert trajectories. Equation (\ref{finalcons}) gives a constraint of value function estimators based on the definition of \emph{perform better}, and equation (\ref{policypre}) provides an off-policy method to optimize policy function regardless of how expert demonstrations perform.

\section{Algorithms with Expert Demonstrations}

Theorem \ref{theo} provides a way to satisfy constraint (\ref{cons}) and update policies $\pi_\theta$ with demonstrations of expert policy $\pi^*$, and does not need reward data sampled from $\pi^*$. In this section, we organize the results in Section 3 in a more piratical way, then we apply the pretraining methods to two of the typical actor-critic RL algorithms, DDPG and ACER.

These actor-critic RL algorithms use neural networks $Q^w(s,a)$ to estimate the state-action value functions of policy, $Q^\pi(s,a)$, where $\pi$ is the is the current learned policy while training, which is a parameterized function, $\pi_\theta$, always in the form of artificial neural networks.

For pretraining processes based on Theorem \ref{theo}, we need an estimator of advantage function for policy $\pi_\theta$, $A^\pi(s,a)$. Based on parameterized policy and state-action value function estimator $Q^w$, we obtain the advantage function estimator $A^{w,\theta}$:

\begin{equation}\label{acte}
A^{w,\theta}(s_t^*,a_t^*)=Q^w(s_t^*,a_t^*)-V^{w,\theta}(s_t^*)
\end{equation}
\begin{equation} \label{ve}
V^{w,\theta}(s_t^*)=\mathbb{E}_{a\sim\pi_\theta(s)}Q^w(s_t^*,a)
\end{equation}

Considering the training processes of DDPG and ACER, at the beginning of the processes the policies are nearly random and estimators $Q^w(s,a)$ are not accurate, since there is little data from simulation. Therefore if there exist some expert demonstrations that perform better than initial policies, we can introduce the data using constraint (\ref{finalcons}), in order to obtain a more accurate estimator $Q^w(s,a)$.

If constraint (\ref{finalcons}) is satisfied, then $Q^w(s,a)$ is accurate enough for the fact that $\pi^*$ \emph{performs better}. Hence we update the estimator with expert demonstrations with the following gradient, in which $[x]_+=x$ if $x> 0$, otherwise is zero:

\begin{equation}\label{wg}
g_Q^* = \nabla_w\left[\mathbb{E}_{s^*_0,a^*_0,...\sim \pi^*}\left[\sum_{t=0}^{\infty}\gamma^t A^{w,\theta}(s^*_t,a^*_t)\right]\right]_+
\end{equation}

From equation (\ref{policypre}), we optimize policy with expert demonstrations. Since expert demonstrations do not contain reward data, we can update policy parameters with a simple policy gradient:

\begin{equation}\label{thg}
g_\pi^* = -\nabla_\theta\mathbb{E}_{s^*_0,a^*_0,...\sim \pi^*}\left[\sum_{t=0}^{\infty}\gamma^t A^{w,\theta}(s^*_t,a^*_t)\right]
\end{equation}

For the reason that $\pi^*$ is not the optimal policy of the MDPs, we only train with expert demonstrations for a limited period of time at the beginning of the training process, to guarantee $\pi^*$ performs better than $\pi_\theta$, hence we call the process \emph{pretraining}.

To pretrain actor-critic RL algorithms like DDPG and ACER, we add gradients $g_Q^*$ and $g^*_\pi$ to the original gradients of the algorithms:

\begin{equation}\label{pretrainq}
g_Q^{pre} = g_Q+\lambda_Q g_Q^*\\   
\end{equation}
\begin{equation}\label{pretrainp}
g_\pi^{pre}=g_\pi+\lambda_\pi g_\pi^* 
\end{equation}

Where $g_Q$ and $g_\pi$ are original gradients of baseline actor-critic RL algorithms, and$g^{pre}_Q$ and $g^{pre}_\pi$ are pretraining gradients for estimator $Q^w$ and parameterized policy function $\pi_\theta$ respectively while pretraining. We introduce expert demonstrations to the base algorithms instead of replacing them, since the state-action value functions are estimated with the baseline algorithms and gradient $g_Q^*$ only makes  $Q^w$ satisfy constraint (\ref{cons}).

\subsection{Pretraining DDPG}

DDPG is a representative off-policy actor-critic deterministic RL algorithm. The algorithm is for continuous action space MDPs, and optimizes the policy using off-policy policy gradient.

Two neural networks are used in DDPG at the same time. One is named critic network, which is the state-action value function estimator $Q^w$, and the other is named actor network, which is the parameterized policy $\pi_\theta$. Since it is an algorithm for deterministic control, the input of the actor network is a state of MDPs, and the output is the corresponding action.

Two neural networks are trained simultaneously, with gradients $g_Q$ and $g_\pi$ respectively. $g_Q$ is based on Bellman equation, and $g_\pi$ is the off-policy policy gradient.

In order to introduce expert demonstrations for pretraining critic network and actor network, we apply (\ref{pretrainq}) and  (\ref{pretrainp}) to pretrain the two neural networks.

\begin{figure}
	\centering
	\includegraphics[width=3.4in]{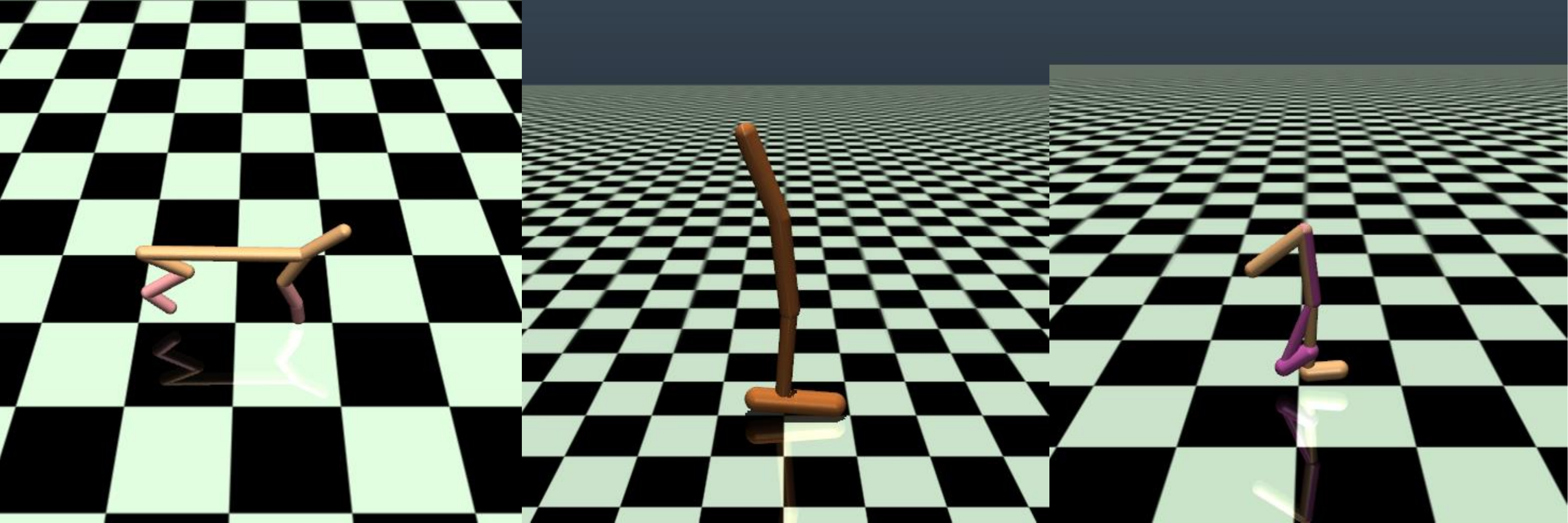}
	\caption{Example screenshots of MuJoCo simulation environments that we attend to experiment on with DDPG as baseline. The tasks are: HalfCheetah (left), Hopper (middle), and Walker2d (right).}
	\label{mujoco}
\end{figure}

\begin{figure*}
	\centering
	\includegraphics[width=7.1in]{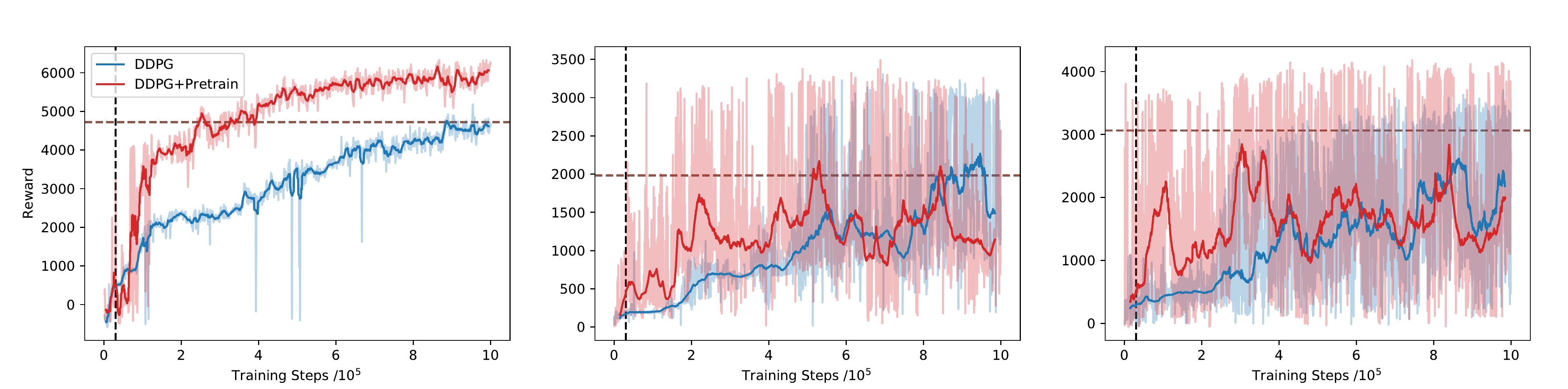}
	\caption{Results of pretraining based on DDPG. The figures each is a different task, and they are respectfully experimented on HalfCheetah (left), Hopper (middle) and Walker2d (right), The vertical dashed black lines represent the points when pretraining end, and the horizontal dashed brown lines represent the average episode reward of expert demonstrations. The transparent blue and red lines are original training results, and the  opaque lines are smoothed lines with sliding windows.}
	\label{ddpg_result}
\end{figure*}

Note that for a deterministic policy $\pi_\theta$, equation (\ref{ve}) becomes $V^{w,\theta}(s)=Q^w(s,\pi_\theta(s))$.

\subsection{Pretraining ACER}

ACER is an off-policy actor-critic stochastic RL algorithm, which modifies the policy gradient to make the process sample efficient. ACER solves both discrete control problems and continuous control problems.

For discrete control problems, a double-output convolutional neural work (CNN) is used in ACER. One output is a softmax policy $\pi_\theta$, and the other is $Q^w$ values. Although $\theta$ and $w$ share most of the parameters, they are updated separately with different gradients. 

For stochastic control problems, a new structure named Stochastic Dueling Networks (SDNs) is used for value function estimation. The network outputs a deterministic value estimation $V^{w}(s)$, and a stochastic state-action value estimation $Q^{w,\theta}(s,a)\sim V^w(s)+A^w(s,a)-\frac{1}{n}\sum_{i=1}^nA^w(s,\dot{a})|_{\dot{a}\sim \pi_\theta}$. Hence equation (\ref{acte}) becomes $A^{w,\theta}(s_t^*,a_t^*)=Q^{w,\theta}(s_t^*,a_t^*)-V^{w}(s_t^*)$.

In ACER, gradient $g_\pi$ is the modified policy gradient, and  $g_Q$ is based on Retrace. Both of the gradients are explained in Section 2.

Policy gradient is estimated using trust region in ACER, but in this paper, we compute pretraining gradients $g^*_Q$ and $g^*_\pi$ directly with expert demonstrations.

\section{Experiments}

We test our algorithms based on DDPG and ACER on various environments, in order to investigate how \emph{simulation efficient} the pretraining methods are. The baselines are DDPG and ACER without pretraining.

Because of the existence of $[x]_+$, $g_Q^*$ defined in (\ref{wg}) could be infinity sometimes. Hence we clip the gradient during pretraining. We set $\lambda_Q$ and $\lambda_\pi=1$ in equations (\ref{pretrainq}) and (\ref{pretrainp}).

The expert policies that generate expert demonstrations are policies trained with baseline algorithms, i.e. DDPG and ACER.

With DDPG as baseline, we apply our algorithm to low dimensional simulation environments using the MuJoCo physics engine \cite{todorov2012mujoco}, and test on tasks with action dimensionality are: HalfCheetah (6D),  Hopper (3D), and Walker2d (6D). These tasks are illustrated in Figure \ref{mujoco}.

All the setups with DDPG as baseline share the same network architecture that compute policies and estimate value functions referring to \cite{Lillicrap2015}. Adam \cite{kingma2014adam} is used for learning parameters and the learning rate of actor network and critic network are respectively $10^{-3}$ and $10^{-4}$. For critic network, $L_2$ weight decay of  $10^{-2}$ is used with $\gamma=0.99$. Both actor network and critic network have 2 hidden layers with 400 and 300 units respectively.

The results of our pretraining method based on DDPG are illustrated  in Figure \ref{ddpg_result}. In the figures, the horizontal dashed brown lines represent the average episode reward of expert demonstrations. It is obvious that the expert demonstrations are not global optimal demonstrations, and in order to guarantee the expert policies \emph{perform better} than learned policies, the pretraining process stops early with 30000 training steps and 60000 simulation steps.

As shown in Figure \ref{ddpg_result}, it is obvious that DDPG with our pretraining method outperforms initial DDPG. Results on HalfCheetah (Figure \ref{ddpg_result} left) is representative and clear, pretraining process gives training a warm start, and after pretraining stops, the performance drops because of the new learning gradient. However, after pretraining, DDPG learns faster than the baseline, hence it outperforms initial DDPG. Although the results of DDPG are unstable on Hopper (Figure \ref{ddpg_result} middle) and Walker2d (Figure \ref{ddpg_result} right), smoothed results indicate that DDPG with pretraining processes learns faster than DDPG.

\begin{figure}
\centering
\includegraphics[width=3.4in]{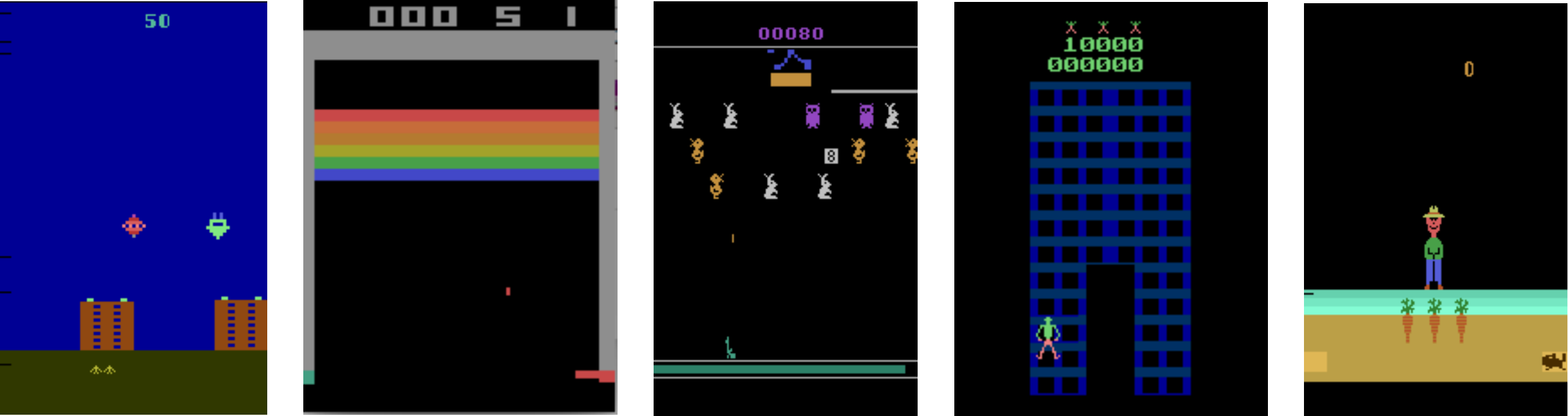}
\caption{Example screenshots of Atari simulation environments that we attend to experiment on with ACER as baseline. The tasks from left to right are: AirRaid, Breakout, Carnival, CrazyClimber and Gopher.}
\label{atari}
\end{figure}

\begin{figure*}
	\centering
	\includegraphics[width=7.1in]{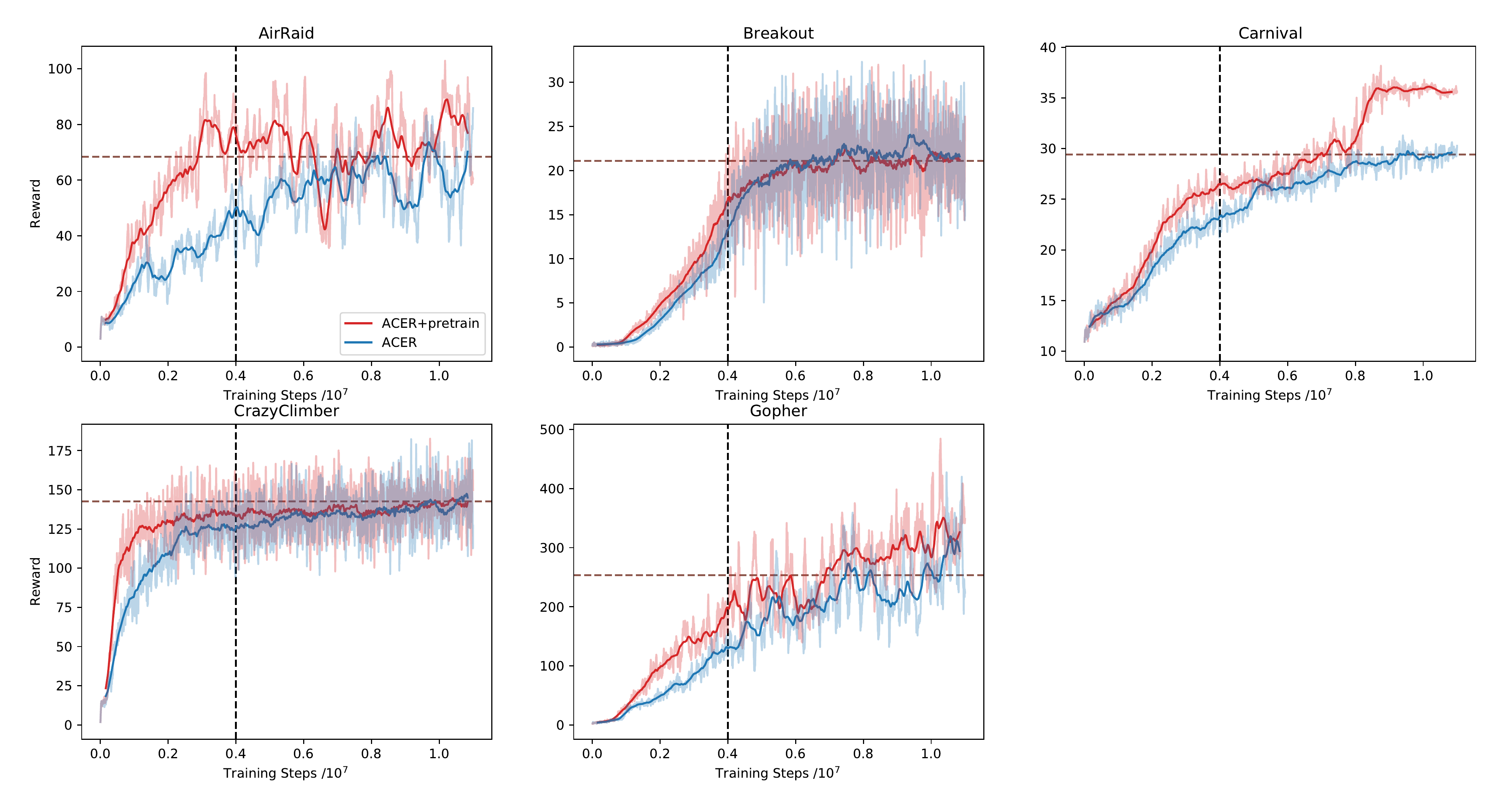}
	\caption{Results of pretraining based on ACER with trust region update. Similar to Figure \ref{ddpg_result}, the vertical dashed black lines are the points when pretraining end, and the horizontal dashed brown lines are the average episode reward of expert demonstrations. The transparent red and blue lines represent the original training results, and the opaque ones are smoothed results with sliding windows.}
	\label{acer_result}
\end{figure*}

With ACER as baseline, we apply our algorithm to image based Atari games. We only tested on discrete control problems with ACER, and the environments we tested on are: AirRaid, Breakout, Carnival, CrazyClimber and Gopher. The environments are illustrated in Figure \ref{atari}.

The experiment settings are similar to \cite{Wang2016}, The double-output network consists of a convolutional layer with 32 $8\times 8$ kernels with stride 4, a convolutional layer with 64 $4\times 4$ kernels with stride 2, a convolutional layer with 64 $3\times 3$ kernels with stride 1, followed by a fully connected layer with 512 units. The network outputs a softmax policy and state-action value Q for every action.

Because of the limitation of memory, each thread of ACER only have a replay memory of 5000 frames, which is the only different setting from \cite{Wang2016}. Entropy regularization with weight 0.001 is also adopted, and the discount factor $\gamma=0.99$, importance weight truncation $c=10$. Trust region updating is used as described in \cite{Wang2016}, and all the settings of trust region update remain the same. ACER without trust region update is not tested in this paper.

The results of our pretraining method based on ACER with trust region update is illustrated in Figure \ref{acer_result}. All of the environments are image based Atari games. All the lines have the same meaning as Figure \ref{ddpg_result}, and it is obvious that ACER with pretraining process outperforms initial ACER. 

Unlike DDPG, the performance of learned policies does not fall after pretraining process ends. This is because for stochastic discrete control, a random policy and a random state-action value estimator always satisfies constraint (\ref{cons}), hence $g_Q^*$ defined in (\ref{wg}) is always zero, and $g_\pi^*$ defined in (\ref{thg}) is policy gradient based on expert demonstrations, similar to original $g_\pi$ from baseline ACER, therefore the performance of learned policies does not fall after pretraining.

Note that learning with expert demonstrations use the same amount of simulation steps as baseline algorithms, our pretraining method is more \emph{simulation efficient} than baselines.

\section{Conclusion}

In this work, we propose an extensive method that pretrains actor-critic reinforcement learning methods. Based on Theorem \ref{theo}, we design a method that takes advantage of expert demonstrations. Our method does not rely on the global optimal assumption of expert demonstrations, which is one of the key differences between our method and IRL algorithms. Our method pretrains policy function and state-action value estimators simultaneously with gradients (\ref{pretrainq}) and (\ref{pretrainp}). With experiments based on DDPG and ACER, we demonstrate that our method outperforms the raw RL algorithms.

One limitation of our framework is that it has to estimate the advantage function for expert demonstrations, and the framework is not suitable for algorithms like A3C \cite{mnih2016asynchronous} and TRPO \cite{Schulman2015} that only maintain a value estimator $V^w(s)$. On the other hand, the fact that expert demonstrations \emph{perform better} is not considered during pretraining of policies (Equation (\ref{thg})). We left these extensions in our future work.

\section*{Acknowledgments}

This work was supported by National Key R\&D Program of China (No. 2016YFB0100901), and National Natural Science Foundation of China (No. 61773231).

\bibliographystyle{named}
\bibliography{ijcai18}

\end{document}